\documentclass[letterpaper, 10 pt, conference]{ieeeconf}  % Comment this line out
\IEEEoverridecommandlockouts                              % This command is only
\makeatletter
\let\NAT@parse\undefined
\makeatother

-mathfrak at 10pt
-mathcal at 10pt

\renewcommand{\baselinestretch}{0.91}

\usepackage[dvipsnames]{xcolor}

\newcommand*\linkcolours{ForestGreen}
\usepackage{optidef} %added by ulises 063020
\usepackage{times}
\usepackage{graphicx}
\usepackage{amsmath}
\usepackage{amsfonts}
\usepackage{amssymb}
\usepackage{amsxtra}
\usepackage{amsthm} % Unnecessary in LNCS
\usepackage{mathrsfs}
\usepackage{stmaryrd}

\usepackage{breakurl}

\usepackage[colorlinks=true,urlcolor=black,linkcolor=blue,citecolor=blue]{hyperref}
\usepackage[capitalize,nameinlink]{cleveref}
\hypersetup{
colorlinks,
linkcolor=\linkcolours,
citecolor=\linkcolours,
filecolor=\linkcolours,
urlcolor=\linkcolours}

\usepackage[linesnumbered,lined,boxed,commentsnumbered]{algorithm2e}

\SetDataSty{textnormal}
\SetCommentSty{textnormal}
\usepackage{algpseudocode} 
\usepackage{listings}
\let\labelindent\relax
\usepackage{enumitem}
\usepackage[labelfont={bf},font=small]{caption}
\usepackage[none]{hyphenat}

\usepackage{mathtools, cuted}

\usepackage{tabularx}

\usepackage{float}

\usepackage{pifont}% http://ctan.org/pkg/pifont

\newcolumntype{Y}{>{\centering\arraybackslash}X}

\usepackage[]{placeins}

\usepackage{placeins}

\usepackage{tikz}

\usepackage[framemethod=tikz]{mdframed}

\usepackage{afterpage}

\usepackage{stfloats}

\usepackage{comment}

\makeatletter
\newcommand{\removelatexerror}{\let\@latex@error\@gobble}
\makeatother

\makeatletter
\def\@endtheorem{\endtrivlist}
\makeatother

\newtheorem{theorem}{Theorem}
\newtheorem{definition}{Definition}
\newtheorem{proposition}{Proposition}

\newtheorem{problem}{Problem}

\newcommand{\ou}{%
  \mathrel{%
    \vcenter{\offinterlineskip
      \ialign{##\cr$<$\cr\noalign{\kern-1.5pt}$>$\cr}%
    }%
  }%
}%

\begin{document}

\title{\LARGE \bf
Certified Vision-based State Estimation for Autonomous Landing Systems using Reachability Analysis
}

\author{Ulices Santa Cruz$^{1}$ and Yasser Shoukry$^{1}$% <-this % stops a space
% begin nowrap
\thanks{
This work was partially sponsored by the NSF awards \#CNS-2002405, \#CNS-2013824, and \#CNS-2313104.
}
% end nowrap
\thanks{$^{1}$Ulices Santa Cruz, and  Yasser Shoukry are with the Department of Electrical Engineering and Computer Science, University of California Irvine, Email: \{usantacr,yshoukry\}@uci.edu}%
}

\maketitle

%%%%%%%%%%%%%%%%%%%%%%%%%%%%%%%%%%%%%%%%%%%%%%%%%%%%%%%%%%%%%%%%%%%%%%%%%%%%%%%%
\begin{abstract}

This paper studies the problem of designing a certified vision-based state estimator for autonomous landing systems. In such a system, a neural network (NN) processes images from a camera to estimate the aircraft's relative position with respect to the runway. We propose an algorithm to design such NNs with certified properties in terms of their ability to detect runways and provide accurate state estimation. At the heart of our approach is the use of geometric models of perspective cameras to obtain a mathematical model that captures the relation between the aircraft states and the inputs. We show that such geometric models enjoy mixed monotonicity properties that can be used to design state estimators with certifiable error bounds. We show the effectiveness of the proposed approach using an experimental testbed on data collected from event-based cameras. 
%Our experimental results show that the proposed framework outperforms NNs trained using state-of-the-art data-driven machine-learning techniques.

\end{abstract}

% \thispagestyle{empty}
% %\pagestyle{empty}
% \pagestyle{plain}

% !TEX root = ../main.tex

%%%%%%%%%%%%%%%%%%%%%%%%%%%%%%%%%%%%%%%%%%%%%%%%%%%%%%%%%%%%%%%%%%%%%%%%%%%%%%%%
\section{INTRODUCTION}
Machine learning models, like deep neural networks, are increasingly used to control dynamical systems in safety-critical applications. These black-box models trained using data are used heavily to process high-dimensional imaging data like LiDAR scanners and cameras to produce state estimates to low-level, model-based controllers. While these deep Neural Networks (NNs) provide empirically accepted results, they lack certified guarantees in terms of their ability to process complex scenes and provide estimates of the location of different objects within the scene. It is then unsurprising the increasing number of reported failures of these deep NNs in building reliable autonomous systems.

In this paper, we will consider the safety of deep neural networks that control aircraft while approaching runways to perform
an autonomous landing. Such a problem enjoys geometric nature that can be exploited to develop a geometrical/physical model of the perception system. Yet, it represents a significant real-world problem of interest to the designers of the autonomous system. In particular, we present a novel neural network-based filter that can process complex scenes along with estimates of the state of the aircraft---computed by unverified complex deep neural networks---and output a state estimate of the aircraft with a certified error bound. That is, akin to the ``control shields'' in the reinforcement learning literature~\cite{ferlez2020shieldnn,alshiekh2018safe}, 
the proposed filter can be thought of as a ``shield'' that can filter out incorrect estimates of the aircraft and replaces them with ones with certified error bounds. In contrast, the correct estimates pass this filter (or shield) unaltered.

\begin{figure}[t!]
\centering
\includegraphics[width=1.0\columnwidth]{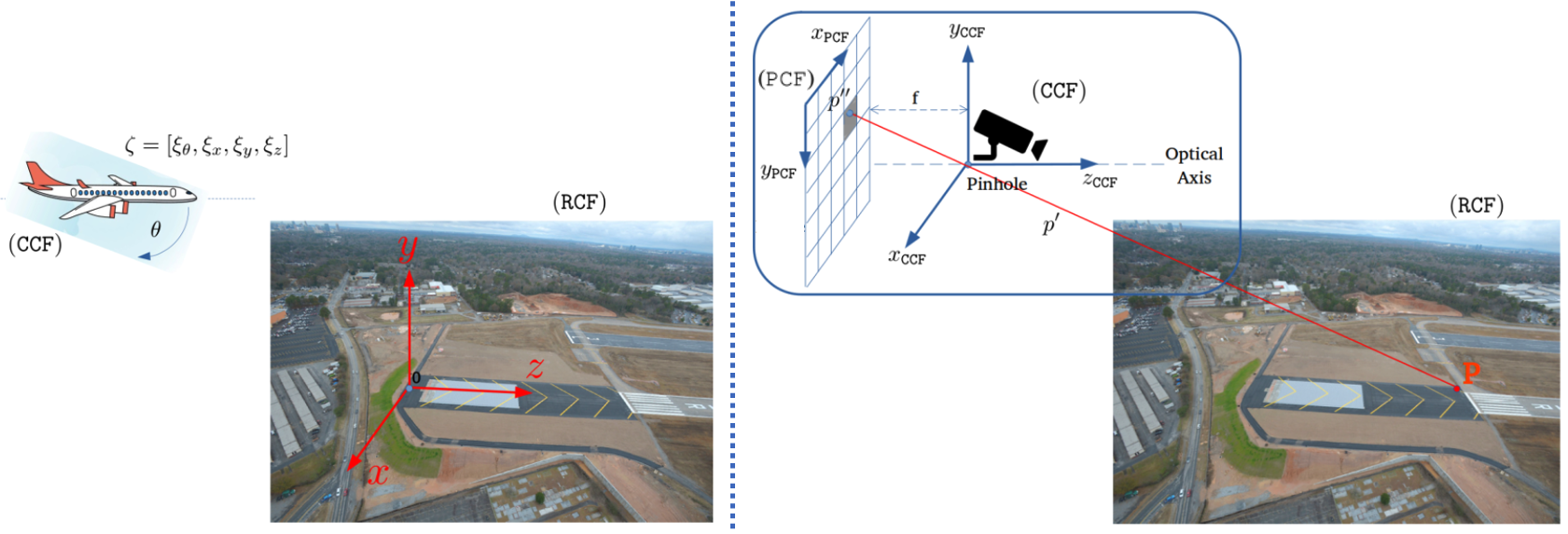}
\caption{Coordinate frames: Runway ($\texttt{RCF}$), Camera ($\texttt{CCF}$) and Pixel ($\texttt{PCF}$).}
\label{coordinates}
\vspace{-5mm}
\end{figure}

A central challenge to designing such a filter is the need to explicitly model the imaging process, i.e., the relation between the system state and the images created by the camera~\cite{SantaCruz2022}. An early result on the application of formal verification for vision-based dynamical systems controlled with neural networks \cite{Sun2019} focused only on the usage of LiDARs. The first steps in formally modeling the imaging process for camera-based systems have been recently studied in~\cite{Mitra2022,katz2022verification,habeeb2023verification}. In particular, the work in~\cite{Mitra2022,katz2022verification} proposes the use of abstractions of the perception system as a formal model of perception. Unfortunately, these abstractions are only tested on a set of samples and lack guarantees in their ability to model the perception system formally. The work in~\cite{habeeb2023verification} extends the notion of imaging-adapted partitions, originally defined for LiDAR images~\cite{Sun2019}, to the notion of image-invariant regions, which are regions within which the captured images are identical. Unfortunately, the work in~\cite{habeeb2023verification} focuses only on simple scenes that can be modeled as a collection of triangles that represent the triangulated faces of objects in the environment. The work in~\cite{talak2023certifiable} considers the problem of estimating the pose of different objects in the scene. Given a partial point cloud of an object, the goal is to estimate the object's pose and provide a certificate of correctness for the resulting estimate. While capable of handling complex objects, the framework in~\cite{talak2023certifiable} is sound but not complete, meaning that if it can identify the object's pose, it will generate a certificate. Still, not all poses of the object will be identified, even if the object of interest exists in the scene. Other techniques include classification that uses targeted inputs with the aim of finding counterexamples that violate safety~\cite{shoouri2021falsification}. However, such techniques do not provide formal guarantees regarding the ability to find all counterexamples.

In this paper, we build on our recent results~\cite{SantaCruz2022} that exploit the geometry of the autonomous landing problem to construct a formal model for the image formation process (a map between the aircraft states and the image produced by the camera). This physics-based formal model is designed such that it can be encoded as a neural network (with manually chosen weights) that we refer to as the Runway Generative Model neural network. Thanks to the recent development in computing the reachable sets of neural networks (the set of all possible outputs of the network)~\cite{tran2020nnv,ferlez2022polynomial,tran2019star}, we can characterize the set of all possible images for the runway. We use such reachability analysis to design novel filters that can remove all the other objects in the scene by matching the spatial and geometrical properties of the runway to those in the computed reachable set. Moreover, as a by-product of this design, the proposed filter identifies the set of possible state estimates of the aircraft. This set of possible state estimates can then be used to cross-check the ones computed by unverified neural network estimators and provide certifiable error bounds on the final state estimate.

\section{Preliminaries} % (fold)
\label{sec:preliminaries}

\subsection{Notation} % (fold)
\label{sub:notation}

We denote by $\mathbb{N}$, $\mathbb{B}$, $\mathbb{R}$ and $\mathbb{R^+}$ the set of natural, Boolean, real, and non-negative real numbers, respectively. %\r{add definition of class K and K infinity functions.}
We use $||x||_{\infty}$ to denote the infinity norm of a vector $x\in\mathbb{R}^n$. We denote by $\mathcal{B}_r(c)$ the infinity norm centered at $c$ with radius $r$, i.e., $\mathcal{B}_r(c) = \{x \in \mathbb{R}^n | || c - x||_{\infty} \le r\}$.  We use the notation $A_{[i,j]}$ to 
denote the element in the $i^\text{th}$ row and $j^\text{th}$ column of $A$. 
Analogously,  the notation $A_{[i,:]}$ denotes the 
$i^\text{th}$ row of $A$, and $A_{[:, j]}$ denotes 
the $j^\text{th}$ column of $A$; when $A$ is a vector instead, both notations 
return a scalar. Let $\mathbf{0}_{n,m}$ be an $(n \times m)$ matrix of zeros, and  
$\mathbf{1}_{n,m}$ be the $(n \times m)$ matrix of ones. 
Finally, the symbols $\oplus$ and $\otimes$ denote element-wise addition and multiplication of matrices.

\subsection{Aircraft State Space}

In this paper, we consider an aircraft landing on a runway. We assume the states of the aircraft to be measured with respect to the origin of the Runway Coordinate Frame (shown in Figure~\ref{coordinates} (left)), where positions are: $\xi_x$ is the axis across runway; $\xi_y$ is the altitude and $\xi_z$ is the axis along the runway. We consider only one angle $\xi_\theta$, representing the pitch rotation around the $x$ axis of the aircraft. The state vector of the aircraft at time $t \in \mathbb{N}$ is denoted by $\xi^{(t)} \in \mathbb{R}^4 = [\xi_\theta^{(t)}, \xi_x^{(t)}, \xi_y^{(t)}, \xi_z^{(t)}]^T$.

\subsection{Runway Parameters}

We consider a runway that consists of two border line segments, $L$ and $R$. Each line segment can be characterized by its start and end point (also measured in the Runway Coordinate Frame) i.e., $L = [(L_x, 0, L_z), (L_x + r_w, 0, L_z + r_{l})]$ and $R = [(R_x, 0, R_z), (R_x + r_w, 0, R_z + r_l)]$ where $r_w$ and $r_l$ refers to the runway width and length (e.g. standard international runways are designed with $r_w = 40$ meters wide and $r_l= 3000$ meters).

\subsection{Camera Model} % (fold)
\label{sub:Camera_Model}

We assume the aircraft is equipped with a monochrome camera $\mathcal{C}$ that produces images of $a \times b$ pixels. Since the camera is assumed to be monochromatic, each pixel in the image $I$ takes a value of 0 or 1. The image produced by the camera depends on the relative location of the aircraft with respect to the runway and the other objects in the scene. In other words, we can model the camera $\mathcal{C}$ as a function that maps aircraft states into images, i.e., $\mathcal{C}: \mathbb{R}^4 \rightarrow \mathbb{B}^{a\times b}$. Although the images created by the camera depend on the runway parameters and the other objects in the scene, we drop this dependence from the notation $\mathcal{C}$ for ease of notation.

We utilize an ideal pinhole camera model~\cite{ma2012invitation} to capture the image formation process of this camera. In general, a point $p$ in the Runway Coordinate Frame ($\texttt{RCF}$) is mapped into a point $p'$ on the Camera Coordinate Frame ($\texttt{CCF}$) using a translation and rotation transformations defined by~\cite{MultipleView}:
\begin{equation}
    \begin{bmatrix}
    p'_{x_\texttt{CCF}} \\
    p'_{y_\texttt{CCF}} \\
    p'_{z_\texttt{CCF}} \\
    1
\end{bmatrix}=\begin{bmatrix}
    1 & 0 & 0 & x\\
    0 & \cos{\theta} & \sin{\theta} & y\\
    0 & -\sin{\theta} & \cos{\theta} & z\\
    0 & 0 & 0 & 1
\end{bmatrix}\begin{bmatrix}
    p_x \\
    p_y \\
    p_z \\
    1
\end{bmatrix}
\end{equation}

\begin{figure}[t!]
\centering
\includegraphics[width=1.0\columnwidth]{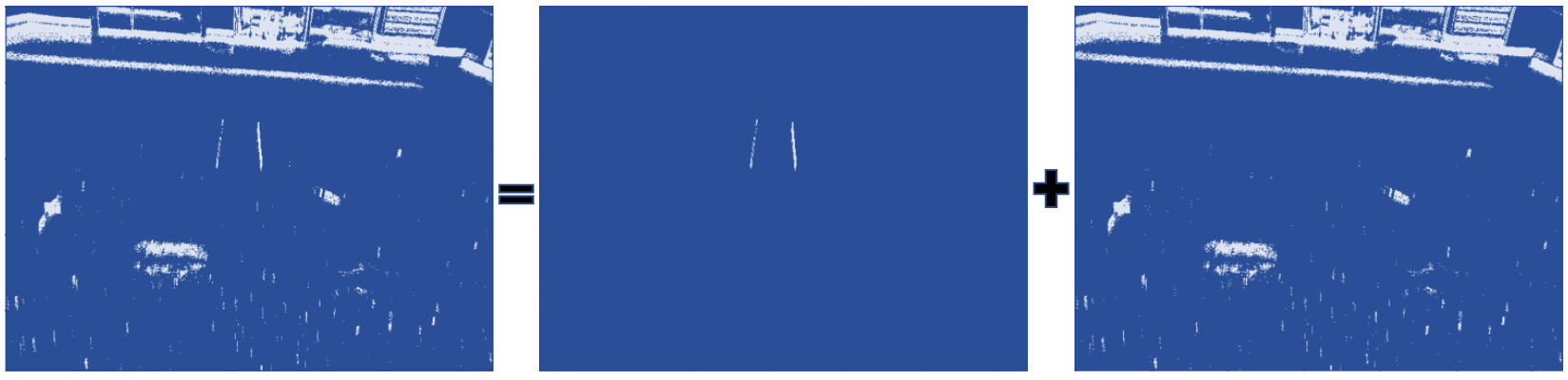}
\caption{Monochromatic images generated using state-of-the-art event-based cameras. The full image $I$ to the left can be decomposed into one that contains only the runway image $I_r$ (center) and the remaining objects/noise $I_n$ (right), i.e., $I = I_r + I_n$.}
\label{fig:image}
\vspace{-5mm}
\end{figure}

The camera then converts the 3-dimensional point $p'$ on the camera coordinate frame into two-dimensional point $p''$ on the Pixel Coordinate Frame ($\texttt{PCF}$) as:
\begin{align}
    p'' = \left(p''_{x_\texttt{PCF}}, p''_{y_\texttt{PCF}} \right) = \left(\left\lfloor \frac{ q_{x_\texttt{PCF}} }{ q_{z_\texttt{PCF}}} \right\rfloor, \left\lfloor\frac{ q_{y_\texttt{PCF}} }{ q_{z_\texttt{PCF}}} \right\rfloor \right)
    \label{eq:pcf_floor}
\end{align} 
where:
\begin{equation}
    \begin{bmatrix}
    q_{x_\texttt{PCF}} \\
    q_{y_\texttt{PCF}} \\
    q_{z_\texttt{PCF}}
\end{bmatrix}=
\begin{bmatrix}
\rho_w & 0 & u_0\\
0 & -\rho_h & v_0\\
0 & 0 & 1
\end{bmatrix}
\begin{bmatrix}
    f & 0 & 0 & 0\\
    0 &  f & 0 & 0\\
    0 & 0 & 1 & 0
\end{bmatrix}
\begin{bmatrix}
    p'_{x_\texttt{CCF}} \\
    p'_{y_\texttt{CCF}} \\
    p'_{z_\texttt{CCF}} \\
    1
\end{bmatrix}
\end{equation}
and $f$ is the focal length of the camera lens, $\textrm{W}$ is the image width (in meters), $\textrm{H}$ is the image height (in meters), $a$ is the image Width (in pixels), $b$ is the image height (in pixels), and $u_0 = 0.5a$, $v_0 = 0.5b$, $\rho_w=\frac{\textrm{a}} {\textrm{W}}$, $\rho_h=\frac{\textrm{b}}{\textrm{H}}$. The values of each pixel in the final image $I$ can be computed directly by checking if the point $p''$ lies within the area of the pixel and assigning 1 to such pixels accordingly~\cite{MultipleView}.

What is remaining is to map the coordinates of $p'' = \left(p''_{x_\texttt{PCF}}, p''_{y_\texttt{PCF}} \right)$ into a binary assignment for the different $a\times b$ pixels. But first, we need to check if $p''$ is actually inside the physical limits of the Pixel Coordinate Frame ($\texttt{PCF}$) by:

\begin{equation}
  \text{visible} =
    \begin{cases}
      \text{yes} & |p''_{x_\texttt{PCF}}| \leq \frac{W}{2} \ \vee \ |p''_{y_\texttt{PCF}}| \leq \frac{H}{2}\\
      \text{no} & \text{otherwise}
    \end{cases}       
\end{equation}
Whenever the point $p''$ is within the limits of $\texttt{PCF}$, then the pixel $I_{[k,l]}$ should be assigned to 1 whenever the index of the pixel matches the coordinates $\left(p''_{x_\texttt{PCF}}, p''_{y_\texttt{PCF}} \right)$, i.e.:
\begin{equation}
  I_{[k,l]} =
    \begin{cases}
      1 & ( p''_{x_\texttt{PCF}} == k-1) \wedge (p''_{y_\texttt{PCF}} == l-1) \wedge \text{visible}\\
      0 & \text{otherwise}
    \end{cases}  
\end{equation}
for $k \in (1,2,3...\textrm{a})$ and $l \in (1,2,3...\textrm{b})$. This process of mapping a point $p$ in the Runway Coordinate Frame $(\mathtt{RCF})$ to a pixel in the image $I$ is summarized in Figure~\ref{coordinates} (right). %As shown in (), equations (1-5) can be rewritten as a change of coordinates and a neural network.

Since the scene contains both a runway and other unknown objects (see Figure~\ref{fig:image}), we define the final image $I \in \mathbb{B}^{a\times b}$ captured by the camera as:
\begin{align}
    I(\xi) = I_r(\xi) + I_n(\xi)
\end{align}
where $I_r \in \mathbb{B}^{a\times b}$ is the image corresponding to the existence of the runway in the scene and $I_n \in \mathbb{B}^{a\times b}$ is the image corresponding to the existence of other objects/noise in the scene.

%Given the image formation process defined in~\eqref{??}-\eqref{??}, we define the images captured by the camera 
%%%%%%%%%%%%%%%%%%%%%%%%%%%%%%%%%%%%%%%%%%%%%%%%%%%%
\subsection{Neural Network Estimator}
We are interested in designing a Neural Network (NN)-based estimator that can process an image $I(\xi) = I_r(\xi) + I_n(\xi)$ to produce an estimate of the aircraft state $\xi$. An $F$-layer NN is specified by composing $F$ layer functions (or just layers). A layer $\omega$ with $\mathfrak{i}_\omega$ inputs and $\mathfrak{o}_\omega$ outputs is specified by a weight matrix $W^{\omega} \in \mathbb{R}^{\mathfrak{o}_\omega \times \mathfrak{i}_\omega}$ and a bias vector $b^{\omega} \in \mathbb{R}^{\mathfrak{o}_\omega}$ as follows:
\begin{equation}
    \label{eq:layer_fnc}
    L_{\theta^{\omega}}: z \mapsto \phi( W^{\omega} z + b^{\omega}), 
\end{equation}
where $\phi$ is a nonlinear function, and $\theta^{\omega} \triangleq (W^{\omega}, b^{\omega})$ for brevity. Thus, an $F$-layer NN is specified by $F$ layer functions $\{L_{\theta^{\omega}} : \omega = 1, \dots, F\}$ whose input and output dimensions are composable: that is, they satisfy $\mathfrak{i}_{\omega} = \mathfrak{o}_{\omega-1}$, $\omega = 2, \dots, F$. Specifically:
\begin{equation}
	\mathcal{NN}(I) = (L_{\theta^{F}} \circ L_{\theta^{F-1}} \circ \dots \circ L_{\theta^{1}})(I).
\end{equation}
%where we drop the index a ReLU NN function with a list of parameters $\theta \triangleq (\theta^{1}, \dots, \theta^{F})$. 
As a common practice, we allow the output layer $L_{\theta^{F}}$ to omit the nonlinear function $\phi$.%, and we use ReLU nonlinearities $\phi(x) = \max(0,x)$.

%%%%%%%%%%%%%%%%%%%%%%%%%%%%%%%%%%%%%%%%%%%%%%%%%%%%
\subsection{Problem Formulation} % (fold)
\label{sub:problem_formulation}

\begin{problem}
\label{prob:main_problem}
    Given an image $I(\xi) = I_r(\xi) + I_n(\xi)$ that contains the projection of a runway and other unknown objects and an estimation error $\epsilon>0$, design a neural network estimator $\mathcal{NN}$ such that $||\mathcal{NN}(I_r + I_n) - \xi||< \epsilon$.
\end{problem}

% !TEX root = ../main.tex

\begin{figure}[t!]
\centering
\includegraphics[width=1.0\columnwidth]{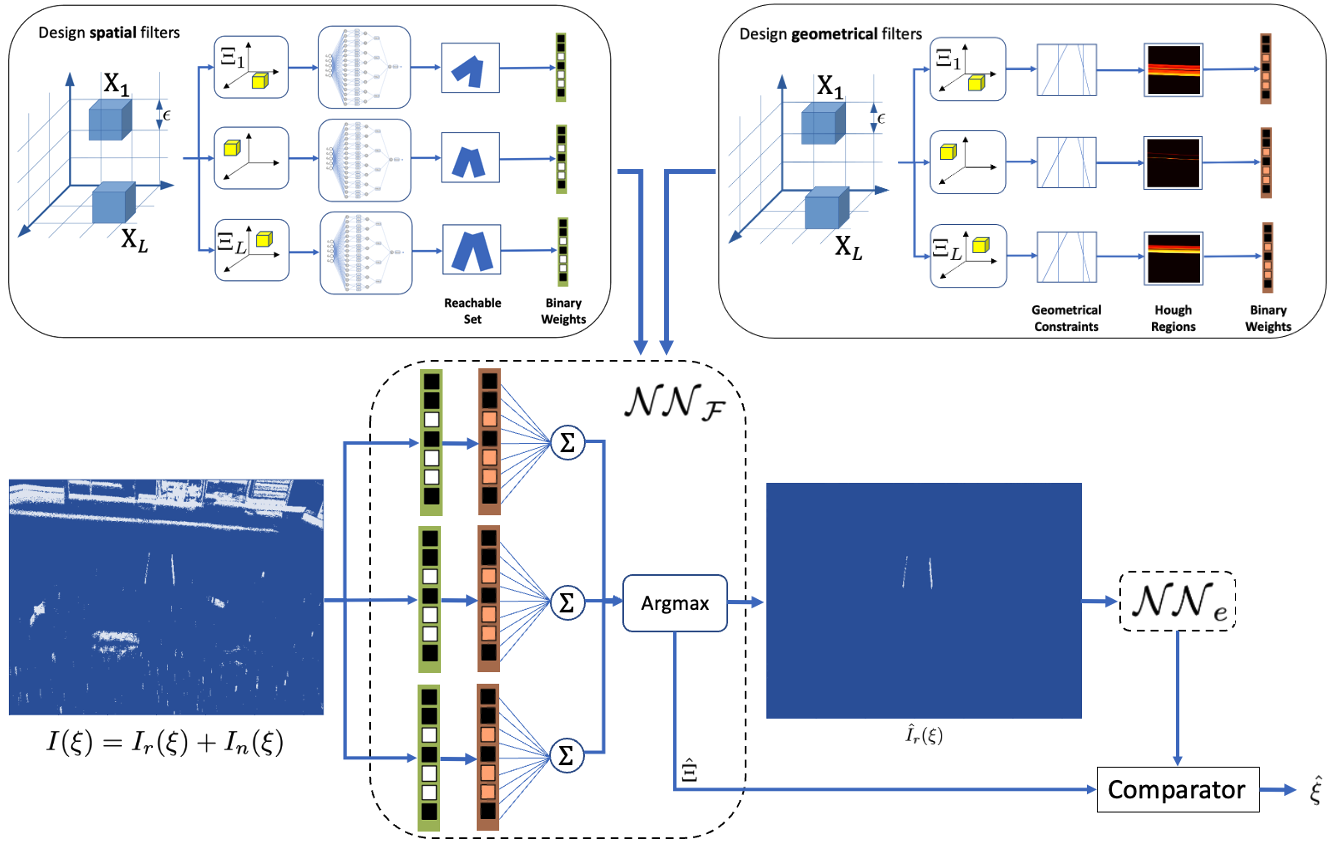}
\caption{Overall main framework elements: Spatial filter, Geometrical filter, $\mathcal{NN}_\mathcal{F}$ and $\mathcal{NN}_e$}
\label{fig:framework}
\vspace{-3mm}
\end{figure}

\section{Framework}
\label{sec:framework}
Classical machine learning approaches to solve Problem 1 entail training neural networks on large labeled data sets that contain different possibilities of runway positions and surrounding objects. Since ensuring the correctness of the resulting NN is challenging, we propose a framework in which we manually design a NN filter $\mathcal{NN}_{\mathcal{F}}$ that is guaranteed to ``filter out'' the noise $I_n$, i.e., $\mathcal{NN}_{\mathcal{F}}(I_r + I_n) = I_r$. Moreover, such a filter $\mathcal{NN}_{\mathcal{F}}$ also computes a certified bound on the possible states of the aircraft $\hat{\Xi}$. The size of this possible set of states $\hat{\Xi}$ is chosen to guarantee the $\epsilon$ bound in Problem 1. The resulting filtered-out image $\mathcal{NN}_{\mathcal{F}}(I_r + I_n)$ is then passed into a neural network estimator $\mathcal{NN}_e$ that is trained using existing techniques in machine learning. Finally, the outcome of $\mathcal{NN}_e$ is checked against the certified bounds $\hat{\Xi}$ to provide the final estimate as:
\begin{align}
    \hat{\xi} = 
    \begin{cases}
    \mathcal{NN}_e\left(\mathcal{NN}_{\mathcal{F}}(I) \right) & \text{if } \mathcal{NN}_e\left(\mathcal{NN}_{\mathcal{F}}(I) \right) \in \hat{\Xi} \\
    \text{center}(\hat{\Xi}) & \text{otherwise}
    \end{cases}
\end{align}
where $\text{center}(\hat{\Xi})$ is well defined whenever the set $\hat{\Xi}$ is a hypercube. In other words, the certified bounds $\hat{\Xi}$ are used to \emph{replace} the incorrect state estimates with ones with guaranteed error bound from within the set $\hat{\Xi}$. This process is depicted in Figure~\ref{fig:framework}. Steps to manually design the NN filter $\mathcal{NN}_{\mathcal{F}}$ and its theoretical guarantees are given in the subsequent subsections.

% To manually design the NN filter $\mathcal{NN}_{\mathcal{F}}$, we partition the state space of $\xi$ into hypercubes $\Xi_1, \ldots \Xi_L$. We then utilize a physics-based generative model to obtain the reach set (set of all possible images) for the runway image $I_r(\xi)$. This reach set is used to design a spatial filter and a geometrical filter

\subsection{Physics-based Generative Model for Runway Images:}
Our prior work in \cite{SantaCruz2022}  developed a physics-based generative model that can generate all possible images containing runways $I_r(\xi)$ based on the physical parameters of the camera $f,\rho_h, \rho_w, v_0, u_0$  (discussed in Section 2). Crucially, this physics-based generative model was shown to be mathematically equal to a change of coordinates $h:\mathbb{R}^4 \rightarrow \mathbb{R}^4$ and a neural network $\mathcal{NN}_r(h(\xi))$ with carefully selected weights and parameters (this network is depicted in Figure~\ref{fig:generative}), i.e.,
$$I_r(\xi) = \mathcal{NN}_r(h(\xi)).$$
The change of coordinates $h$ maps the state of the aircraft into the projections of the endpoints of the lines L and R on the Pixel Coordinate Frame (PCF). For the sake of brevity, we defer the details of $h$ and $\mathcal{NN}_r$ to Appendix~\ref{appendix:gen_model}, and we refer the reader to \cite{SantaCruz2022} for detailed analysis of the correctness of this generative model.

\begin{figure}[t!]
\centering
\includegraphics[width=0.8\columnwidth]{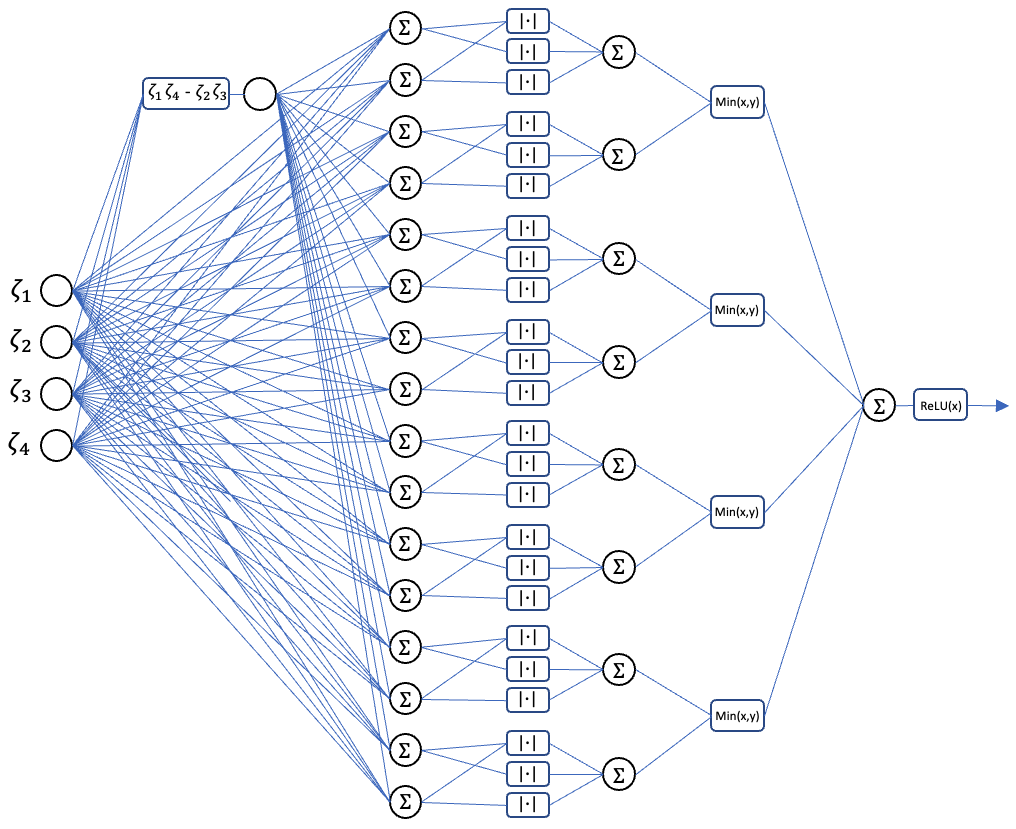}
\caption{Physics-based generative model for runway images $I_r(\xi)$ captured mathematically as a neural network~\cite{SantaCruz2022}.}
\label{fig:generative}
\vspace{-5mm}
\end{figure}

%\r{@Ulices, we need a picture that explains the change of coordinates $h$.}

%%%%%%%%%%%%%%%%%%%%%%%%%%%%%%%%%%%%%%%%%%%%%%%%%%%%%%%%%%%%%%%%%%%%%%%%%%%%%%%%%%%%%%%%

\subsection{Design of spatial filters using output reachability analysis}
Given a partitioning parameter $\delta$, we partition the state space $\Xi \subset \mathbb{R}^4$ into L regions $\Xi_1, \ldots, \Xi_L$ such that each $\Xi_i$ is an infinity-norm ball with radius $\delta$. For each of these partitions, we aim to design a spatial filter that matches the spatial properties of the runway images that can be produced by states within such a partition. To that end, consider the following filter $\mathcal{S}^{\Xi_i} \in \mathbb{B}^{a \times b}$ defined as:
\begin{align}
    \mathcal{S}^{\Xi_i} = \bigotimes_{h(\xi) \in \Xi_i} I_r(\xi) = \bigotimes_{h(\xi) \in \Xi_i} \mathcal{NN}_r(h(\xi)).
    \label{eq:S_filter}
\end{align}
Recall that all images $I_r(\xi)$ are monochromatic (i.e., each pixel can take only a value of 0 or 1), then the following result follows directly from the definition above.

\begin{proposition}
Consider the filter $\mathcal{S}^{\Xi_i}$ defined in~\eqref{eq:S_filter}. The following holds:
\begin{align}
   (i) &\xi \in \Xi_i, \; \forall \xi \in \Xi_i.[I_n(\xi) \otimes \mathcal{NN}_r(h(\xi)) = \mathbf{0}_{a,b}] \notag \\ & \qquad \qquad \qquad \Longrightarrow [I_r(\xi) + I_n(\xi)] \otimes \mathcal{S}^{\Xi_i} \!=\! I_r(\xi) \\
   (ii) & \xi \notin \Xi_i, \; I_n(\xi) \notin \mathcal{I}_r^{\Xi_i} \notag \\ & \quad \qquad \qquad \Longrightarrow [I_r(\xi) + I_n(\xi)] \otimes \mathcal{S}^{\Xi_i} \ne I_r(\xi)
\end{align}
where $\mathcal{I}_r^{\Xi_i} = \{I_r(\xi) \in \mathbb{B}^{a \times b} |  h(\xi) \in \Xi_i\}$.
\end{proposition}
Note that the condition $\forall \xi \in \Xi_i.[I_n(\xi) \otimes \mathcal{NN}_r(h(\xi)) = \mathbf{0}_{a,b}]$ is equivalent to $I_n(\xi) \otimes \mathcal{S}^{\Xi^i} = \mathbf{0}_{a,b}$. That is, the filter $\mathcal{S}^{\Xi_i}$ is capable of removing all noise in the image as long as the noise image $I_n(\xi)$ does not affect pixels that are $\delta/\rho_w$ away from the runway image $I_r(\xi)$. Figure~\ref{fig:spatial_filter} shows an example of such a filter.
Specifically, equations (11)-(12) imply that the filter will accurately process the filtered image, provided that the noise does not resemble the pattern of runways. Additionally, the filter must be applied to the specific region corresponding to the state responsible for generating such a runway. Furthermore, it is reasonable to assume that as we increase the geometric complexity of the runway, the likelihood of noise resembling runway patterns diminishes. In other words, the more intricate the entity we are examining, the safer it is to rely on our assumptions.

\begin{figure}[t!]
\centering
\includegraphics[width=0.8\columnwidth]{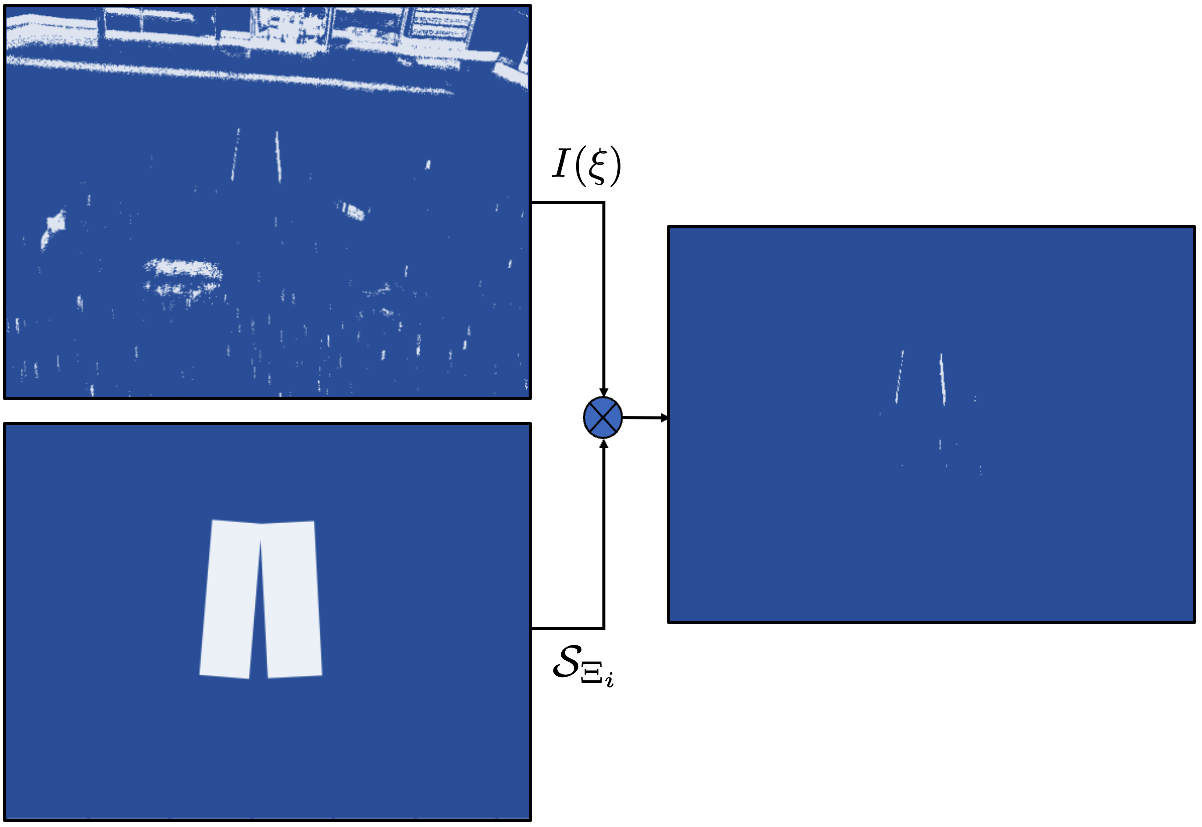}
\caption{Spatial filtering focuses attention on different regions.}
\label{fig:spatial_filter}
\vspace{-3mm}
\end{figure}

% the set of runway images generated by states in this partition $\mathcal{I}_r^{\Xi}$, i.e.,
% \begin{align}
%     \mathcal{I}_r^{\Xi} &= \{I_r(\xi) \in \mathbb{B}^{q \times q} |  h(\xi) \in \Xi\}.
% \end{align}

% Our first building block is to a linear spatial filter $\mathcal{S}_\Xi \in \mathbb{B}^{q \times q}$ such that:
% \begin{align}
%    h(\xi) \in \Xi \quad \Longleftrightarrow \quad [I_r(\xi) + I_n(\xi)] \otimes \mathcal{S}_\Xi = I_r (\xi)
% \end{align}
% In other words, this spatial filter $\mathcal{S}$ does not distort the runway portion of the image while 

What is remaining is to provide an algorithm that can compute the filter $\mathcal{S}^{\Xi_i}$ for each partition $\Xi_i$. Thanks to the fact that the physics-based generative model $\mathcal{NN}_r(\xi)$ is captured as a neural network, one can use output reachability algorithms to compute an overapproximation of the reach set (set of all possible images) for the runway image $I_r(\xi)$. For that end, we leverage Mixed-monotonicity reachability analysis of neural networks~\cite{Meyer2022} by leveraging the following result:
% \r{@Ulices, place details of the reach set analysis here}\\

% To this end, we compute the reachable set of each partition $\Xi_i$, by using its $\text{center}(\Xi_i)$ and its corresponding size $\delta$. To do it, we leverage the following method:

\begin{proposition}(from \cite{Meyer2022})
Given a neural network $\mathcal{NN}: \mathbb{R}^\mathfrak{i} \rightarrow \mathbb{R}^\mathfrak{o} $ and an interval $[\underline{J}, \overline{J}] \subseteq \mathbb{R}^{\mathfrak{o} \times \mathfrak{i}}$ bounding the derivative of $\mathcal{NN}$ for all input $\zeta \in [\underline{\zeta},{\overline{\zeta}}]$. Let us denote the center of the interval as $J^*$ and for each output dimension $i \in \{1,...,\mathfrak{o} \}$, define input vectors $\underline{\zeta}_{[i,:]}, \overline{\zeta}_{[i,:]} \in  \mathbb{R}^{\mathfrak{i}} $ and a row vector $\alpha^i \in  \mathbb{R}^{1 \times \mathfrak{i}}$ such that for all $j \in \{1,...,\mathfrak{i} \}$ the following holds:
\begin{align}
    &(\underline{\psi}_{[i,j]}, \overline{\psi}_{[i,j]}, \alpha_{[i,j]}) = \; \nonumber
   \\ 
    &\qquad \qquad \begin{cases}
      (\underline{\zeta}_{[:,j]},\overline{\zeta}_{[:,j]},min(0,\underline{J}_{[i,j]})) & \text{if $J^*_{[i,j]} \geq 0$}\\
      (\overline{\zeta}_{[:,j]},\underline{\zeta}_{[:,j]},max(0,\overline{J}_{[i,j]})) & \text{if $J^*_{[i,j]} \leq 0$}\\
    \end{cases}
\end{align}

Then for all neural network input $\zeta \in [\underline{\zeta},\overline{\zeta}]$ and $i \in \{1,...,\mathfrak{o} \}$, we have:
\begin{align}
    \mathcal{NN}(\zeta)_{[i,:]} \in [\mathcal{NN}(\underline{\psi}_{[i,:]}-\alpha_{[i,:]}(\underline{\psi}_{[i,:]}-\overline{\psi}_{[i,:]})), \; \nonumber
   \\ 
   \mathcal{NN}(\overline{\psi}_{[i,:]}+\alpha_{[i,:]}(\underline{\psi}_{[i,:]}-\overline{\psi}_{[i,:]}))]
\end{align}
\end{proposition}

To implement the method in Proposition 2, we define the input vectors as $\overline{\zeta} =\text{center}(\Xi_i)+\frac{\delta}{2}$ and $\underline{\zeta} =\text{center}(\Xi_i)-\frac{\delta}{2}$. Additionally, we compute the bounds on the Jacobian matrix of the neural network $\mathcal{NN}_r$ to find the bounds $[\underline{J}, \overline{J}]$. Details about obtaining such bounds are given in Appendix~\ref{sec:jacobian}. These bounds on the output of $\mathcal{NN}_r$ identifies which pixels are equal to zero for all the images generated by the states in each $\Xi_i$, which can be used to compute the filters in~\eqref{eq:S_filter}.

% need to find the correct jacobian $J$ of the neural network $\mathcal{NN}_r(h(\xi))$ (we defer the details of the jacobian computations to Appendix~\ref{appendix:gen_model}).

%%%%%%%%%%%%%%%%%%%%%%%%%%%%%%%%%%%%%%%%%%%%%%%%%%%%%%%%%%%%%%%%%%%%%%%%%%%%%%%%%%%%%%%%
\subsection{Design of Geometric Filters using Hough Transform}

The spatial filters $\mathcal{S}^{\Xi_i}$ can be used to focus the attention on different regions of the state space. Although these filters provide a guarantee of the output of the filter that satisfies $\xi \in \Xi_i$ it does not provide any guarantee on the output of the filters for which $\xi \notin \Xi_i$. Therefore, it is necessary to augment the spatial filters with another filter that aims to detect whether the output follows the geometrical structure of the runway images.
%However, these filters do not provide a guarantee that the runway has been detected. It is necessary to impose a structure to be found, specifically, we aim to identify well-defined line segments that are created as a function of the vehicle states. 
To achieve this, consider the following filters:
\begin{align}
    \mathcal{H}^{\Xi_i}(I) =
    \begin{cases} 
    1 & \text{if } \exists \xi \in \Xi_i \text{ such that } I = \mathcal{NN}_r(h(\xi))\\
    0 & \text{otherwise}
    \end{cases} 
    \label{eq:H_filter}
\end{align}
Such filter can be efficiently computed using the classical Hough-space transformation \cite{Szeliski22}. In this transformation, a straight line is represented by a normal line that passes through the origin and is orthogonal to that straight line. The equation of the normal line is given by $\rho = \zeta_1\cos{(\theta)} + \zeta_2\sin{\theta}$, where $\rho$ is the length of the normal line and $\theta$ is the angle between the normal line and the x-axis of the Pixel Coordinate Frame. By using the projections of the endpoints of the runway lines edges obtained from $h(\xi)=[\zeta_1,\zeta_2,\zeta_3,\zeta_4]$ as $P_1 = (\zeta_1, \zeta_2)$ and $P_2 = (\zeta_3, \zeta_4)$, we can solve for $\theta$ and $\rho$ for the generated image as:
\begin{align}
    \theta &= \tan^{-1}{(\frac{\zeta_1-\zeta_3}{\zeta_4-\zeta_2})}\\
    \rho &= \zeta_1 \cos{(\theta)} +\zeta_2 \sin{(\theta)}
\end{align}
Given a partition $\Xi_i$, we can obtain the range of $\rho$, $\theta$ for all runway images as follows. First, recall that each partition $\Xi_i$ is an infinity ball with a radius equal to $\delta$ around a center point $\text{center}(\Xi_i) \in \mathbb{R}^4$. The two points $P_1 = (\text{center}(\Xi_i)_{[1]}, \text{center}(\Xi_i)_{[2]})$ and $P_2 = (\text{center}(\Xi_i)_{[3]}, \text{center}(\Xi_i)_{[4]})$ represent 2-dimensional points in the Pixel Coordinate Frame that corresponds to the center of $\Xi_i$ (see Figure \ref{fig:hough_geo} for illustration). Following the 2-dimensional geometry of the Pixel Coordinate Frame, it is direct to show that:
%
% Fortunately, the range for $\rho$, $\theta$ that we need to search for each line can be directly obtained from the partition region, using $\text{center}(\Xi_i)$ and $\delta$ (recall that each partition $\Xi_i$ is an infinity ball with a radius equal to $\delta$). In this case $\theta_{min}< \theta < \theta_{max}$ where:
%
\begin{align}
(\zeta^{c_i}_1, \zeta^{c_i}_2, \zeta^{c_i}_3, \zeta^{c_i}_4) &= \text{center}(\Xi_i) \label{eq:hough_set_1}\\
\theta^{\Xi_i}_{max} &= 
\begin{cases}
    \tan^{-1}{(\frac{\zeta^{c_i}_1-\zeta^{c_i}_3+2\delta}{\zeta^{c_i}_4-\zeta^{c_i}_2+2\delta})},& \text{if } \frac{\zeta^{c_i}_4-\zeta^{c_i}_2}{\zeta^{c_i}_3-\zeta^{c_i}_1} > 0\\
    \tan^{-1}{(\frac{\zeta^{c_i}_1-\zeta^{c_i}_3+2\delta}{\zeta^{c_i}_4-\zeta^{c_i}_2-2\delta})},              & \text{otherwise}
\end{cases} \\
\theta^{\Xi_i}_{min} &= 
\begin{cases}
    \tan^{-1}{(\frac{\zeta^{c_i}_1-\zeta^{c_i}_3-2\delta}{\zeta^{c_i}_4-\zeta^{c_i}_2-2\delta})},& \text{if } \frac{\zeta^{c_i}_4-\zeta_2}{\zeta^{c_i}_3-\zeta^{c_i}_1} > 0\\
    \tan^{-1}{(\frac{\zeta^{c_i}_1-\zeta^{c_i}_3-2\delta}{\zeta^{c_i}_4-\zeta^{c_i}_2+2\delta})},              & \text{otherwise}
\end{cases} \\
\rho^{\Xi_i}_{min} &= 
    \underline{b}_{\delta}\frac{\sqrt{{1+m^2}}}{m+\frac{1}{m}} \\
\rho^{\Xi_i}_{max} &= 
    \overline{b}_{\delta}\frac{\sqrt{{1+m^2}}}{m+\frac{1}{m}}
    \label{eq:hough_set_2}
\end{align}
where $m, \overline{b}_{\delta}, \underline{b}_{\delta}$ are defined in Appendix~\ref{sec:ranges}.

\begin{figure}[t!]
\centering
\includegraphics[width=1.0\columnwidth]{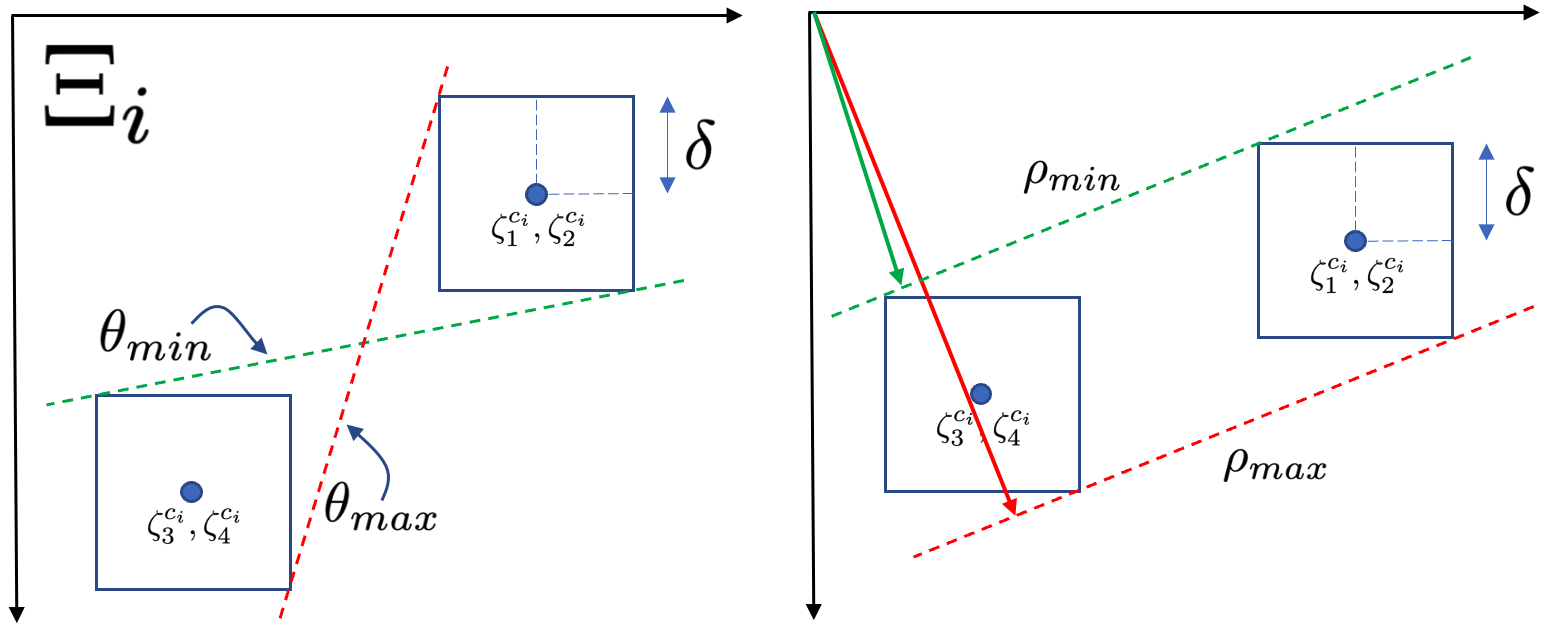}
\caption{Feasible range of angles and distances in Hough Space.}
\label{fig:hough_geo}
\vspace{-3mm}
\end{figure}

Equations~\eqref{eq:hough_set_1}-\eqref{eq:hough_set_2} define the reachable set of the runway images within the Hough space (the $\rho-\theta$ space). 
%We defer the details of $\rho^{\Xi_i}_{min}$ and $\rho^{\Xi_i}_{max}$ to Appendix~\ref{appendix:gen_model}. 
Moreover, the discretization introduced in the Pixel Coordinate Frame (the flooring operation in~\eqref{eq:pcf_floor}) introduces a discretization over the range of $\rho$ and $\theta$ computed by equations~\eqref{eq:hough_set_1}-\eqref{eq:hough_set_2} which existing implementations of Hough transformation algorithms take into account. We denote by $\mathcal{L}^{\Xi_i} = \{(\rho^{\Xi_i}_{max},\theta^{\Xi_i}_{\max}), \ldots, (\rho^{\Xi_i}_{min},\theta^{\Xi_i}_{min})\}$ the discrete set of the allowable values of $\rho$ and $\theta$ within the partition $\Xi_i$. 
For each possible $(\rho_j, \theta_j)$ in $\mathcal{L}^{\Xi_i}$, we define the filter $\mathcal{R}^{\Xi_i}(\rho_j, \theta_j) \in \mathbb{B}^{a\times b}$ as:
\begin{align}
    \mathcal{R}^{\Xi_i}(\rho_j, \theta_j)_{[k,l]} = 
    \begin{cases}
        1 & \text{if} \ l-1<-\frac{\cos{\theta_j}}{\sin{\theta_j}}k + \frac{\rho_j}{\sin{\theta_j}}<l\\
        0 & \text{otherwise}
    \end{cases}
\end{align}
%\forall k,l \in WP
%That is, the filter $\mathcal{R}^{\Xi_i}(\rho, \theta) \in \mathbb{B}^{a\times b}$ passes the pixels that corresponds to lines with
%
% Now, we can define a Reach Set on the Hough space using equations (), such Set can be used for the structure detection.
%
For each filter $\mathcal{R}^{\Xi_i}(\rho_j, \theta_j)$ we can define a mismatching score that computes how far the input image $I$ is from the expected output of this filter as:
\begin{align}
    \mathcal{M}(I, \mathcal{R}^{\Xi_i}(\rho_j, \theta_j)) = \bigg\Vert -I \oplus \mathcal{R}^{\Xi_i}(\rho_j, \theta_j) \bigg\Vert_1
\end{align}
That is, $\mathcal{M}$ is equal to zero whenever the input image $I$ matches exactly the line represented by $\mathcal{R}^{\Xi_i}(\rho_j, \theta_j)$ and non-zero otherwise. Finally, we can implement the filter $\mathcal{H}^{\Xi_i}$ in~\eqref{eq:H_filter} as:
\begin{align}
\mathcal{H}^{\Xi_i}(I) = 
\begin{cases}
    1 & \text{if } \arg\min\{\mathcal{M}(I, \mathcal{R}^{\Xi_i}(\rho^{\Xi_i}_{max}, \theta^{\Xi_i}_{max})), \ldots,
    \\
    & \qquad \qquad \quad \mathcal{M}(I, \mathcal{R}^{\Xi_i}(\rho^{\Xi_i}_{min}, \theta^{\Xi_i}_{min})) \} = 0 \\
    0 & \text{otherwise}
\end{cases}
\label{eq:H_filter_final}
\end{align}
In other words, the filter $\mathcal{H}^{\Xi_i}$ produces 1 whenever any of the filters $\mathcal{R}^{\Xi_i}(\rho^{\Xi_i}_{max}, \theta^{\Xi_i}_{max}), \ldots, \mathcal{R}^{\Xi_i}(\rho^{\Xi_i}_{min}, \theta^{\Xi_i}_{min})$ were able to match its input image. The following proposition follows directly from the definition of $\mathcal{H}^{\Xi_i}(I)$ above.
\begin{proposition}
    Consider the filter $\mathcal{H}^{\Xi_i}$ defined in~\eqref{eq:H_filter_final}. The following holds:
    \begin{align}
        \mathcal{H}^{\Xi_i}(I) = 1 \Longleftrightarrow \exists \xi \in \Xi_i \text{ such that } I = \mathcal{NN}_r(h(\xi))
    \end{align}
\end{proposition}

% Hence, checking the filter in~\eqref{eq:H_filter} is equivalent to identifying all the lines in the image $I$ and checking if their $\theta$ and $\rho$ matches 

% \b{@YS: add a comment that Hough Transform can not be used on the original image directly since it detects lines not line segments but thanks to the spatial filters, it can not detect line segments within the expected length of the runway}.

%%%%%%%%%%%%%%%%%%%%%%%%%%%%%%%%%%%%%%%%%%%%%%%%%%%%%%%%%%%%%%%%%%%%%%%%%%%%%%%%%%%%%%%%
\subsection{Design of the NN filter $\mathcal{NN}_{\mathcal{F}}$}
The final filter $\mathcal{NN}_F$ consists of processing the images $I$ using all the spatial filters $\mathcal{S}^{\Xi_1}, \ldots, \mathcal{S}^{\Xi_l}$ followed by the geometric filters $\mathcal{H}^{\Xi_1}, \ldots, \mathcal{H}^{\Xi_l}$. Finally, the filter $\mathcal{NN}_F$ identifies the partition $\hat{\Xi}$ for which the geometric filter returns $1$ to produce its final outputs as follows:
\begin{align}
    \hat{\Xi} &= \arg\max\{\mathcal{H}^{\Xi_1}(I \otimes \mathcal{S}^{\Xi_1}), \ldots, \mathcal{H}^{\Xi_l}(I \otimes \mathcal{S}^{\Xi_l}) \} \label{eq:NNf_1}\\
    \hat{I}_r &= I \otimes \mathcal{S}^{\hat{\Xi}}
    \label{eq:NNf_2}
\end{align}
The following result captures the correctness of the $\mathcal{NN}_F$.
\begin{theorem}
Consider a noisy image $I(\xi) = I_r(\xi) + I_n(\xi)$, a partitioning of the state space $\Xi$ into infinity balls of radius $\delta$ namely $\Xi_1, \ldots, \Xi_l$. Denote by $\Xi^*$ the partition for which the aircraft state $\xi$ belongs, i.e., $h(\xi) \in \Xi^*$. 
Under the following assumptions:
\begin{align}
    (i) & I_n(\xi) \notin \{\mathcal{NN}_r(h(\xi)) \; | \; h(\xi) \in \Xi \} \label{eq:thm1_ass1}\\
    (ii) & \forall \xi \in \Xi^*.[I_n(\xi) \otimes \mathcal{NN}_r(h(\xi)) = \mathbf{0}_{a,b}] \label{eq:thm1_ass2}
\end{align}
then the following holds: 
\begin{align}
    &\hat{\Xi} = \Xi^* \label{eq:thm1_1}\\
    &\hat{I}_r = I_r(\xi) \label{eq:thm1_2} \\
    &\Vert \xi - \hat{\xi} \Vert \le 4 L_h \delta \qquad \forall \hat{\xi} \in \hat{\Xi} \label{eq:thm1_3}
\end{align}
where $(\hat{\Xi}, \hat{I}_r) = \mathcal{NN}_F(I(\xi))$  and $L_h$ is the Lipschitz constant of $h^{-1}$.
\label{th:correctness}
\end{theorem}
\begin{proof}
We start by proving~\eqref{eq:thm1_1} as follows. For the sake of contradiction, we assume that there exists a partition $\Xi^\dagger \ne \hat{\Xi}$ such that which the aircraft state $\xi$ satisfies $h(\xi) \in \Xi^\dagger$. It follows from Proposition 1 and assumptions~\eqref{eq:thm1_ass1} and~\eqref{eq:thm1_ass2} that:
$$ I \otimes \mathcal{S}^{\hat{\Xi}} \ne I_r(\xi), \qquad I \otimes \mathcal{S}^{\Xi^\dagger} = I_r(\xi) $$ 
and hence Proposition 3 entails that:
$$ \mathcal{H}^{\hat{\Xi}} (I \otimes \mathcal{S}^{\hat{\Xi}}) = 0, \qquad \mathcal{H}^{\Xi^\dagger} (I \otimes \mathcal{S}^{\Xi^\dagger}) = 1 $$
Nevertheless, this contradicts the fact that:
$$ \hat{\Xi} = \arg\max\{\ldots, \mathcal{H}^{\hat{\Xi}} (I \otimes \mathcal{S}^{\hat{\Xi}}), \ldots, \mathcal{H}^{\Xi^\dagger} (I \otimes \mathcal{S}^{\Xi^\dagger}), \ldots \}$$
which proves that $h(\xi) \in \hat{\Xi}$.

Equation~\eqref{eq:thm1_2} follows directly from~\eqref{eq:thm1_1} and Proposition 1. Similarly, equation~\eqref{eq:thm1_3} follows from the fact that the partition $\hat{\Xi}$ is an infinity ball of radius $\delta$ and hence for any $\hat{\xi} \in \hat{\Xi}$:
\begin{align*}
    &\Vert h(\xi) - h(\hat{\xi}) \Vert_\infty = \Vert h(\xi) + \text{center}(\hat{\Xi}) - \text{center}(\hat{\Xi}) -  h(\hat{\xi}) \Vert_\infty
    \\
    & \qquad \le \Vert h(\xi) - \text{center}(\hat{\Xi}) \Vert_\infty + 
    \Vert \text{center}(\hat{\Xi}) -  h(\hat{\xi}) \Vert_\infty \\
    & \qquad  \le 2 \delta
\end{align*} 
Hence from the relation between the 2-norm and the infinity norm, we conclude that:
$$\Vert h(\xi) - h(\hat{\xi}) \Vert \le \sqrt{4} \Vert h(\xi) - h(\hat{\xi}) \Vert_\infty \le 4 \delta $$
from which we conclude that:
$$ \Vert \xi - \hat{\xi} \Vert \le 4 L_h \delta$$
which concludes the proof.
\end{proof}

Before we conclude this section, it is essential to interpret the assumptions~\eqref{eq:thm1_ass1} and~\eqref{eq:thm1_ass2} in Theorem~\ref{th:correctness}. In particular, the assumption in~\eqref{eq:thm1_ass1} entails that the noise $I_n$ can not be generated using the runway generative model $\mathcal{NN}_r$. In other words, this assumption ensures that the noise does not look like a runway and hence only one image of the runway exists in the scene. The assumption in~\eqref{eq:thm1_ass2} asks that the pixels that are $\delta$ close to the runway are not affected by the noise. It is crucial to note that assumption~\eqref{eq:thm1_ass2} is required to be satisfied in $\Xi^*$ only and does not affect other partitions. 

% %%%%%%%%%%%%%%%%%%%%%%%%%%%%%%%%%%%%%%%%%%%%%%%%%%%%%%%%%%%%%%%%%%%%%%%%%%%%%%%%%%%%%%%%
% \subsection{Theoretical Guarantees}

% \begin{align}
%     I(\xi) &= I_r(\xi) + I_n(\xi) = \mathcal{NN}_r(h(\xi))\\
%     \mathcal{I}_r^{\Xi} &= \{I_r(\xi) \in \mathbb{B}^{a \times b} |  \xi \in \Xi\}\\
%     %\text{supp}(I)
% \end{align}

% !TEX root = ../main.tex

\section{Experimental Evaluation}

We present the results of a vision-based aircraft landing system that uses a target runway. 
%The state vector of the aircraft at time $t$ is denoted by $\xi^{(t)} = [\xi_\theta^{(t)}, \xi_x^{(t)}, \xi_y^{(t)}, \xi_z^{(t)}]^T$ (in meters), where $\xi^{(t)}_x=0$. 
We consider two runway segments, $L = [(L_x, 0, L_z), (L_x, 0, L_z + r_{l})]$ and $R = [(R_x, 0, R_z), (R_x, 0, R_z + r_l)]$ where $R_x=0.1$, $L_x=-0.1$, $R_z=0$, $L_z=0$, $r_{l}=0.3$ (in meters).

To generate monochromatic images, we utilized the SilkyEvCam event-based camera with a resolution of $640 \times 480$ pixels, a focal length of $8$ mm, and a pixel size of $15 \ \mu m \times 15 \ \mu m$. We measured the ground-truth states of the vehicle using Vicon motion capture cameras to track optical markers attached to the camera envelope, and the centroid of the camera was defined as the camera coordinate frame (CCF) origin. Similarly, we defined the runway target as the runway coordinate frame (RCF) from which all measurements were made. %All experiments were conducted on an Apple M1 Pro processor with 32 GB of RAM.

We partitioned the state space $\Xi \subset \mathbb{R}^4$ into 27 regions $\Xi_1, \ldots, \Xi_{27}$ using a partitioning parameter $\delta = 0.1$. These regions correspond to the range of states $[\xi_y \times \xi_z \times \xi_\theta] = [0.8,1]\times[1.6,1.8]\times[0.5,0.7]$ (we fix $\xi_x = 0$ in our experiments). We then implemented the runway generative model neural network $\mathcal{NN}_r$ for a resolution of $640 \times 480$ pixel images, the filter $\mathcal{NN}_f$, and the corresponding application of the spatial $\mathcal{S}^{\Xi_i}$ and geometric filters $\mathcal{H}^{\Xi_i}$ on all partitions to create the binary weights needed using PyTorch libraries. This process took approximately $20$ minutes per partition, resulting in a total of approximately $9$ hours to generate the neural network weights for all $27$ partitions using an Apple M1 Pro processor with 32 GB of RAM.

Next, the filter $\mathcal{NN}_F$ was used to process images collected from the SilkyEvCam event-based camera. We operated the camera for several minutes resulting in a total of $1320$ images using 25 frames per second. Figure~\ref{fig:test1} and Figure~\ref{fig:test2} show two instances of the images collected and processed during our experiments. As seen from the two figures, the scene contains one runway and several objects, and noisy pixels. The neural network $\mathcal{NN}_F$ is used to filter these images and remove all objects except for the runway. Figure~\ref{fig:test1}(right) and Figure~\ref{fig:test2}(right) show the outputs of the 4 different spatial filters $\mathcal{S}^{\Xi_i}$. As can be observed in the two figures, the result of these filters focuses the attention on specific segments of the scene. Some of these filtered images contain the runway (or segments of it) while others contain only parts of the noise image $I_n$. Next, we execute the geometric filters $\mathcal{H}^{\Xi_i}$ to identify the images that match the geometric structure of the runways. We highlight the partition with the smallest mismatch score $\mathcal{M}$ with a green box in Figure~\ref{fig:test1} and Figure~\ref{fig:test2}. In particular, in Figure~\ref{fig:test1}, the output corresponding to partition 1 contains leads to the smallest mismatch score while partition 24 corresponds to the one with the smallest mismatch score in Figure~\ref{fig:test2}.

%loaded the binary weights and proceeded to perform its inference on image 1 (Test $1$) taking approximately $37$ sec. to return the correct filtered runway located in partition $\Xi_1$ (shown in Figure \ref{fig:test1}). 

Finally, we used off-the-shelf algorithms to process the filtered image and produce the final state estimate. For the test reported in Figure~\ref{fig:test1}, the resulting state error is 0.0777 while for the test reported in Figure~\ref{fig:test2} the resulting error is 0.045, both are below the threshold of $4\delta L_h$ and hence no further processing is required.

% keypoint detection to find the start and ending pixels of the detected runway segments and used the centers of these pixels on $h^{-1}(\zeta)$ to obtain the state estimate $\hat{\xi}=[0,1.03,1.68,0.6]$ of the vehicle, with an error $||\hat{\xi}-\xi||_2=0.07$. A similar procedure was followed for (Test $2$) (shown in Figure \ref{fig:test2}).

\begin{figure}[t!]
\centering
\includegraphics[width=0.99\columnwidth]{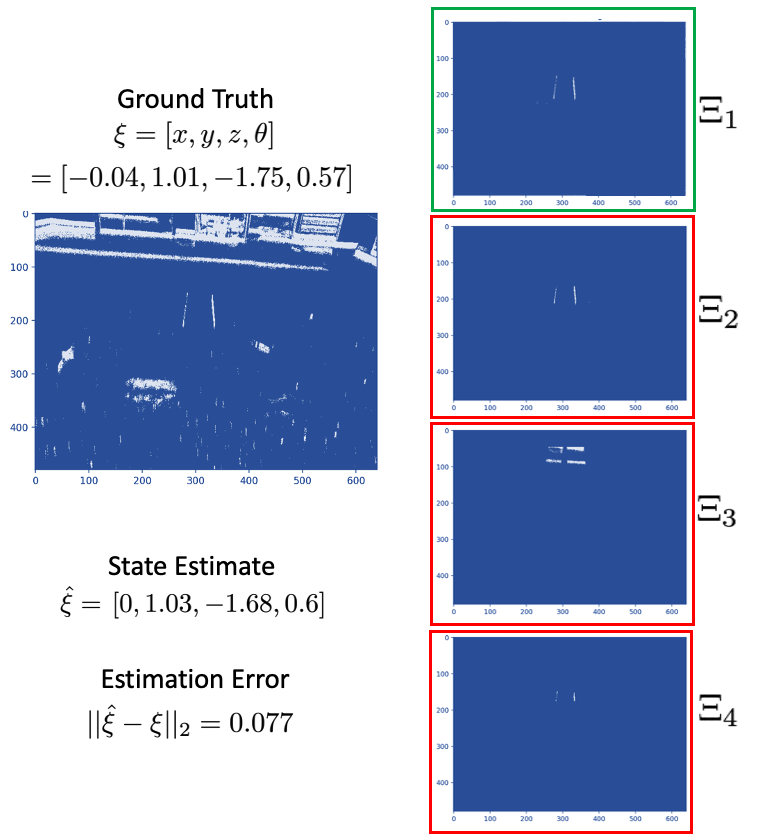}
\caption{Test 1: Framework application on image \#1 delivers correct filtered runway (in Green) found on Partition \#1.}
\label{fig:test1}
\vspace{-3mm}
\end{figure}
\begin{figure}[t!]
\centering
\includegraphics[width=0.99\columnwidth]{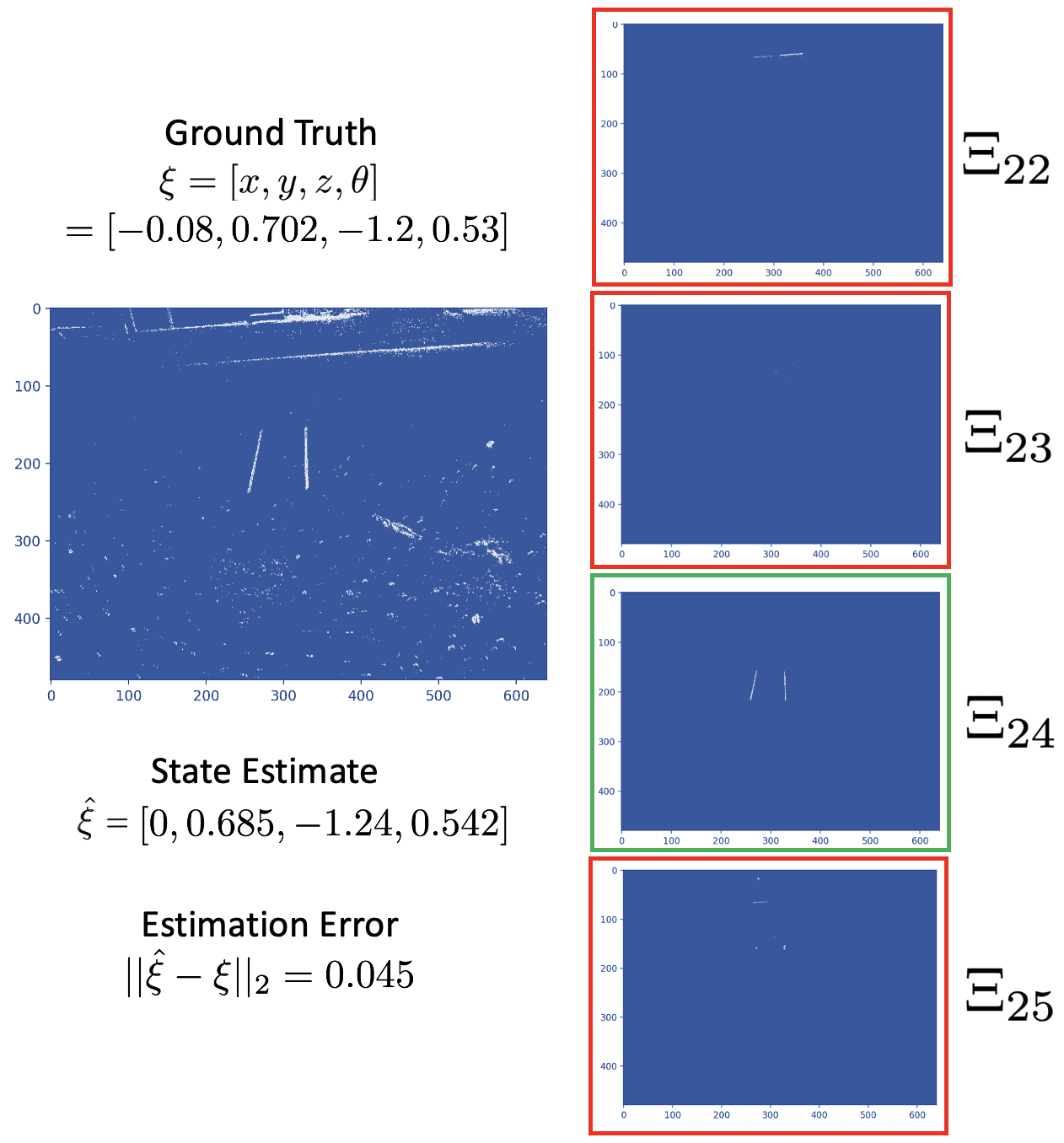}
\caption{Test 2: Framework application on image \#2 delivers correct filtered runway (in Green) found on Partition \#24.}
\label{fig:test2}
\vspace{-3mm}
\end{figure}

Additionally, for comparison purposes, we applied an off-the-shelf standard Hough transformation-based filter that can discover line segments in the scene with the aim to identify the runway without our proposed filter. Figure~\ref{fig:test3} shows the output of the standard Hough transformation-based filter when operated on the same input image used in Figure~\ref{fig:test1}. The dashed lines in Figure~\ref{fig:test3} correspond to the line segments that were detected by the standard filter. As can be appreciated from Figure~\ref{fig:test3}, the standard filter leads to several false line detections that do not match the runway due to the noise and the other objects in the scene. Fortunately, our proposed filter does not suffer from such an issue and comes with provable guarantees.% in terms of its ability to filter the image and bound the state estimation error as captured by Theorem~\ref{th:correctness} and shown in the previous experiment.

% without geometrical constraints on image $1$, it can be appreciated multiple false line detections (shown in Figure \ref{fig:test3}).

\section{Conclusion}

In this study, we introduce a novel neural network-based filter designed to process complex scenes for obtaining state estimates from images. This filter not only produces state estimates but also provides certified error-bound estimates, essentially acting as a protective shield against inaccurate estimations while preserving the accurate ones. Our approach incorporates three key components.
Firstly, we developed a physics-based generative model tailored to generate runway images, leveraging the specific physical characteristics of the camera. 
Secondly, we introduced spatial filters to ensure a consistent match with the spatial attributes of runway images across diverse states within a given partition. This is achieved through the application of mixed-monotonicity reachability analysis.
Lastly, we integrated geometric filters to evaluate whether the output aligns with the typical geometric structure of runway images, employing the robust Hough transform. These geometric filters work synergistically with the spatial filters, allowing us to zero in on distinct regions within the state space.
The theoretical validity of our network filter is demonstrated under mild assumptions. Our results demonstrate that the neural network effectively discriminates against unwanted objects, focusing solely on the runway. Additionally, geometric filters further enhance the identification process.
Comparison with a standard Hough transformation-based filter underscores the effectiveness of our proposed system. The standard filter tends to produce false detections due to noise and other objects in the scene, whereas our filter consistently yields reliable outcomes with provable guarantees.

\begin{figure}[t!]
\centering
\includegraphics[width=0.8
\columnwidth]{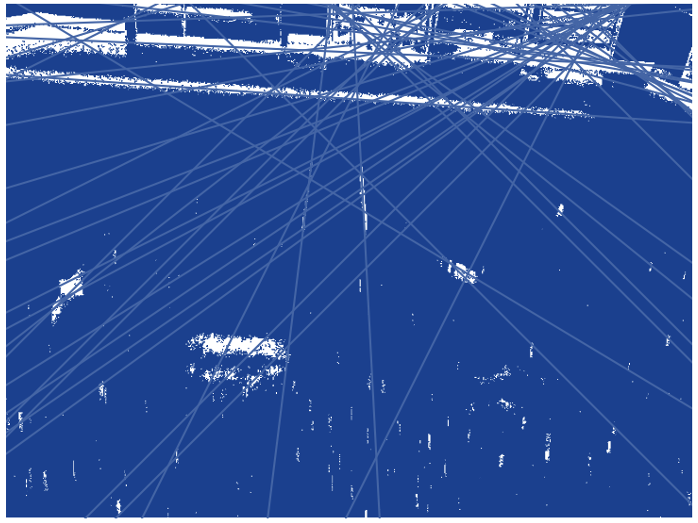}
\caption{Filtering using only Hough filter without geometrical constraints. %Left: Original Image (Test 1). Right: False detections
}
\label{fig:test3}
\vspace{-3mm}
\end{figure}

% \section{Conclusions}

%%%%%%%%%%%%%%%%%%%%%%%%%%%%%%%%%%%%%%%%%%%%%%%%%%%%%%%%%%%%%%%%%%%%%%%%%%%%%%%%

\bibliographystyle{ieeetr}
\bibliography{bibliography}

% !TEX root = ../main.tex

%%%%%%%%%%%%%%%%%%%%%%%%%%%%%%%%%%%%%%%%%%%%%%%%%%%%%%%%%%%%%%%%%%%%%%%%%%%%%%%%
\section{Appendix}
\label{sec:appendix}

\subsection{Physics-Based Generative Model} % (fold)
\label{appendix:gen_model}
The generative model is made of two components: A change of coordinates $h$ and a ReLU-based Neural Network for Line Generation  $\mathcal{NN}_r$ as captured by the following definitions.

\begin{definition}[Change of Coordinates]

We define the change of coordinates as:
\begin{align}
    \zeta &= h  (\xi) =
    \begin{bmatrix} 
    \zeta_1 \\~\\ \zeta_2 \\~\\ \zeta_3 \\~\\ \zeta_4 
    \end{bmatrix}
    =
     \begin{bmatrix} 
    \rho_wf\frac{L_x+\xi_x}{L_z cos(\xi_\theta)+\xi_z}+u_0 
    \\~\\
      -\rho_hf\frac{L_z sin(\theta)+\xi_y}{L_z cos(\xi_\theta)+\xi_z}+v_0 
    \\~\\
       \rho_wf\frac{(L_x + r_w) +\xi_x}{(L_z + r_L) cos(\xi_\theta)+\xi_z}+u_0
    \\~\\
       -\rho_hf\frac{(L_z + r_L) sin(\xi_\theta)+\xi_y}{ (L_z + r_L) cos(\xi_\theta)+\xi_z}+v_0
    \end{bmatrix}
\end{align}
where $f,\rho_h, \rho_w, v_0, u_0$ are the camera physical parameters as defined in Section 2. 
\end{definition}
In other words, the pair $(\zeta_1, \zeta_2)$ is the projection of the start point of the runway $(L_x,0,L_z)$ onto the Pixel Coordinate Frame $\texttt{PCF}$ (while ignoring the flooring operator for now). Similarly, the pair $(\zeta_3, \zeta_4)$ is the projection of the endpoint of the runway $(L_x+r_w,0,L_z+r_L)$ onto the $\texttt{PCF}$ frame. Indeed, we can define a similar set of variables for the other line segment of the runway, $R$. We refer to the new state space as $\Xi$.

%\subsection{Runway Generative Model}

\begin{definition}[Runway Generative Model]

We define the Runway Generative Model $\mathcal{NN}_r$ of $a \times b$ pixels as: 
\begin{equation}
I_r(\xi) = \mathcal{NN}_r(h(\xi))
\end{equation}
for simplicity let's consider $\zeta = h(\xi)$, then:
\begin{equation}
   \mathcal{NN}_r(\zeta) = ReLU(\phi_{Min}(\phi_{Abs}(\phi_{Lin}(\zeta,\phi_{Det}(\zeta)))))
\end{equation}
where:
\begin{equation}
\zeta_{det} = \phi_{Det}(\zeta)=  \zeta_1 \zeta_4 - \zeta_2 \zeta_3
\end{equation}
\begin{equation}
     \zeta^1 = \phi_{Lin}(\zeta,\zeta_{det}) = W[\zeta_1,\zeta_2,\zeta_3,\zeta_4, \zeta_{det}]^T
\end{equation}
\begin{equation}
    \zeta^2= \phi_{Abs}(\zeta^1)= \begin{bmatrix}  -|\zeta^1_{[1]}+\zeta^1_{[2]}|+|\zeta^1_{[1]}|+|\zeta^1_{[2]}|\\
    \vdots\\
    -|\zeta^1_{[15]}+\zeta^1_{[16]}|+|\zeta^1_{[15]}|+|\zeta^1_{[16]}|\\
    \end{bmatrix}
\end{equation}

\begin{equation}
    \zeta^3 = \phi_{Min}(\zeta^2) = \begin{bmatrix}  \min(\zeta^2_{[1]},\zeta^2_{[2]})\\
\vdots \\
    \min(\zeta^2_{[7]},\zeta^2_{[8]})\\
    \end{bmatrix}
\end{equation}
where $\phi_{Det}: \mathbb{R}^4 \rightarrow \mathbb{R}$, $\phi_{Lin}: \mathbb{R}^5 \rightarrow \mathbb{R}^{16\times q}$, $\phi_{Abs}:\mathbb{R}^{16\times q} \rightarrow \mathbb{R}^{8\times q}$, $\phi_{Min}:\mathbb{R}^{8\times q} \rightarrow \mathbb{R}^{4\times q}$ and $ReLU:\mathbb{R}^{4\times q} \rightarrow \mathbb{R}^{q}$, where $q=a \times b$ is the number of pixels in the image to be generated, weight matrix $W \in \mathbb{R}^{\mathfrak{16} \times \mathfrak{5}}$ contains fixed weights fully decribed by camera parameters \cite{SantaCruz2022}. We refer to $I_r(\xi)$ as the image containing solely the runway defined by $\zeta$.

\end{definition}

\subsection{Bounds on the Jacobian Matrix of the Runway Generative Model}
\label{sec:jacobian}
%\begin{definition}[Bounds on the Jacobian Matrix of the Runway Generative Model]

We compute the bound on the Jacobian matrix of the Runway Generative Model as follows:
\begin{equation}
J_{\phi_{Det}} = \begin{bmatrix} \zeta_4 \\ -\zeta_3 \\  -\zeta_2 \\ \zeta_1 \end{bmatrix}
\end{equation}

\begin{equation}
J_{\phi_{Lin}} = \begin{bmatrix} 
    W_{[1,1]} & W_{[2,1]} & \dots \\
    \vdots & \ddots & \\
    W_{[1,5]} &        & W_{[16,5]}
    \end{bmatrix}
\end{equation}

\begin{equation}
J_{\phi_{Abs}}=  \begin{bmatrix} 
    -\frac{\zeta^1_{[1]}+\zeta^1_{[2]}}{|\zeta^1_{[1]}+\zeta^1_{[2]}|} + \frac{\zeta^1_{[1]}}{|\zeta^1_{[1]}|} & \dots & 0 \\
    -\frac{\zeta^1_
    {[1]}+\zeta^1_{[2]}}{|\zeta^1_{[1]}+\zeta^1_{[2]}|} + \frac{\zeta^1_{[2]}}{|\zeta^1_{[2]}|} & \dots & 0 \\
    0 & \dots & \vdots \\
    0 & \dots & 0 \\
    \vdots & \ddots & 0 \\
    0 & \dots  & -\frac{\zeta^1_{[15]}+\zeta^1_{[16]}}{|\zeta^1_{[15]}+\zeta^1_{[16]}|} + \frac{\zeta^1_{[15]}}{|\zeta^1_{[15]}|}\\
    0 & \dots  & -\frac{\zeta^1_{[15]}+\zeta^1_{[16]}}{|\zeta^1_{[15]}+\zeta^1_{[16]}|} + \frac{\zeta^1_{[16]}}{|\zeta^1_{[16]}|}
    \end{bmatrix}    
\end{equation}
Hence, $J_{\phi_{Abs}}$ bounds are:
\begin{equation}
\underline{J}_{\phi_{Abs}} =
     \begin{bmatrix} 
    -2 & 0 & \dots \\
    -2 & 0 & \dots \\
    0 & -2 & \dots \\
    0 & -2 & \dots \\
    \vdots & \ddots & \\
    0 &        & -2\\
    0 &        & -2
    \end{bmatrix} 
\ ; \
\overline{J}_{\phi_{Abs}} =\begin{bmatrix} 
    2 & 0 & \dots \\
    2 & 0 & \dots \\
    0 & 2 & \dots \\
    0 & 2 & \dots \\
    \vdots & \ddots & \\
    0 &        & 2\\
    0 &        & 2
    \end{bmatrix}
\end{equation}

Such bounds are well defined, provided that all inputs $\zeta_1,\zeta_2,\zeta_3,\zeta_4 > 0$, which is the case for all practical applications.\\

$J_{\phi_{Min}}$ bounds are:

\begin{equation}
  \underline{J}_{\phi_{Min}} =\mathbf{0}_{8,4} \ ; \ \overline{J}_{\phi_{Min}} = \mathbf{1}_{8,4}
\end{equation}

Such bounds come from the fact that:
\begin{equation}
    \frac{\partial (Min(\zeta_i,\zeta_j))}{\partial \zeta_i} =
    \begin{cases}
      0 & \text{$\zeta_i \leq \zeta_j$}\\
      1 & \text{otherwise}\\
    \end{cases}
\end{equation}

% \end{definition}

\subsection{Geometric filter ranges $\rho$}
\label{sec:ranges}
% \begin{definition}[Geometric filter ranges $\rho$]

We compute the orthogonal distance from the origin to the line made by
farthest and closest corners of the partition:

\begin{equation}
    m = \frac{\zeta^{c_i}_4-\zeta^{c_i}_2}{\zeta^{c_i}_3-\zeta^{c_i}_1}
\end{equation}

\begin{equation}
    \overline{b}_{\delta} = (\zeta_2^{c_i}+\delta)-m(\zeta_1^{c_i}+\delta)
\end{equation}

\begin{equation}
    \underline{b}_{\delta} = (\zeta_2^{c_i}-\delta)-m(\zeta_1^{c_i}-\delta)
\end{equation}

% \end{definition}

%\begin{figure}[t!]
%\centering
%\includegraphics[width=0.8\columnwidth]{SEC_appendix/GenerativeMod2.png}
%\caption{Physics-based generative model for runway images $I_r(\xi)$ captured mathematically as a neural network~\cite{SantaCruz2022}.}
%\label{fig:generative}
% \vspace{-10mm}
%\end{figure}

\end{document}